\documentclass{article}

\usepackage{amssymb, amsmath, amsopn}
\usepackage{times,mathptmx}

 \usepackage{mathptmx}
 \usepackage{helvet}
 \usepackage{courier}
 \usepackage{makeidx}
 \usepackage{multicol}
 \usepackage{footmisc}
 \usepackage[justification=centering]{caption}
 \usepackage{amssymb, amsmath}
 \usepackage{graphicx}
 \usepackage{caption}
 \usepackage{subfig}
 \usepackage{xspace}
 \usepackage{color}
 \usepackage{array}
 \newcolumntype{P}[1]{>{\centering\arraybackslash}p{#1}}
 \usepackage{verbatim}
 \usepackage{imakeidx, epsfig,url}
\usepackage[linesnumbered,lined,boxed,commentsnumbered]{algorithm2e}
\usepackage{xspace}
\usepackage[numbers,sort&compress,longnamesfirst,sectionbib]{natbib}
\usepackage{balance}
\usepackage{amsthm}
\usepackage{dsfont}
\usepackage{algorithmic}

\newtheorem{theorem}             {Theorem}
\newtheorem{lemma}      [theorem]{Lemma}

\newcommand{\gsemo}{Global SEMO\xspace}
\newcommand{\demo}{DEMO\xspace}
\newcommand{\popBased}{Diverse Population-Based EA\xspace}

\newcommand{\ab}{\hspace{0.125em}}                        
\newcommand{\ie}{\hbox{i.\ab e.}\xspace}                  

\newcommand{\ignore}[1]{}

\title{Parameterized Analysis of Multi-objective Evolutionary Algorithms and the Weighted Vertex Cover Problem}

\author{
Mojgan Pourhassan\\
mojgan.pourhassan@adelaide.edu.au\\
Optimisation and Logistics\\
School of Computer Science\\
The University of Adelaide\\
Adelaide, Australia
\and
 Feng Shi\\
 fengshi@csu.edu.cn\\
 School of Information Science and Engineering\\
 Central South University\\
 Changsha 410083, P.R. China
\and
 Frank Neumann\\
 frank.neumann@adelaide.edu.au\\
 Optimisation and Logistics\\
 School of Computer Science\\
 The University of Adelaide\\
 Adelaide, Australia
}

\begin{document}

\maketitle

\begin{abstract}
A rigorous runtime analysis of evolutionary multi-objective optimization for the classical vertex cover problem in the context of parameterized complexity analysis has been presented by Kratsch and Neumann~\cite{Frank2013VertexCover}. In this paper, we extend the analysis to the weighted vertex cover problem and provide a fixed parameter evolutionary algorithm with respect to $OPT$, where $OPT$ is the cost of the the optimal solution for the problem. Moreover, using a diversity mechanisms, we present a multi-objective evolutionary algorithm that finds a $2-$approximation in expected polynomial time and introduce a population-based evolutionary algorithm which finds a $(1+\varepsilon)-$approximation in expected time $O(n\cdot 2^{\min \{n,2(1- \varepsilon)OPT \}} + n^3)$.
\end{abstract}


\section{Introduction}
The area of runtime analysis has provided many rigorous new insights into the working behaviour of bio-inspired computing methods such as evolutionary algorithms and ant colony optimization~\cite{BookNeuWit,Auger11,ncs/Jansen13}. In recent years, the parameterized analysis of bio-inspired computing has gained additional interest~~\cite{Frank2013VertexCover,DBLP:conf/ppsn/KratschLNO10,Sutton2012makespan,DBLP:journals/ec/SuttonNN14}. Here the runtime of bio-inspired computing is studied in dependence of the input size and additional parameters such as the solution size and/or other structural parameters of the given input.

One of the classical problems that has been studied extensively in the area of runtime analysis is the classical NP-hard vertex cover problem. Here, an undirected graph is given and the goal is to find a minimum set of nodes $V'$ such that each edge has at least one endpoint in $V'$.
Friedrich et al.~\cite{Friedrich2007VertexCover} have shown that the single-objective evolutionary algorithm (1+1)~EA can not achieve a better than trivial approximation ratio in expected polynomial time. Furthermore, they have shown that a multi-objective approach using \gsemo gives a factor $O(\log n)$ approximation for the wider classes of set cover problems in expected polynomial time. Further investigations regarding the approximation behaviour of evolutionary algorithms for the vertex cover problem have been carried out in \cite{DBLP:journals/ec/FriedrichHHNW09,DBLP:journals/tec/OlivetoHY09}. Edge-based representations in connection with different fitness functions have been investigated in~\cite{DBLP:conf/foga/JansenOZ13,DBLP:conf/gecco/PourhassanGN15} according to their approximation behaviour in the static and dynamic setting.
Kratsch and Neumann~\cite{Frank2013VertexCover} have studied evolutionary algorithms and the vertex cover problem in the context of parameterized complexity. They have shown that \gsemo, with a problem specific mutation operator is a fixed parameter evolutionary algorithm for this problem and finds $2-$approximations in expected polynomial time. Kratsch and Neumann~\cite{Frank2013VertexCover} have also introduced an alternative mutation operator and have proved that \gsemo using this mutation operator finds a $(1+\varepsilon)-$approximation in expected time  $O(n^2\log n + OPT\cdot n^2+ n\cdot 4^{(1-\varepsilon)OPT})$.  Jansen et al.~\cite{DBLP:conf/foga/JansenOZ13} have shown that a $2$-approximation can also be obtained by using an edge-based representation in the (1+1)~EA combined with a fitness function formulation based on matchings.

To our knowledge all investigations so far in the area of runtime analysis consider the (unweighted) vertex cover problem. In this paper, we consider the weighted vertex cover problem where in addition weights on the nodes are given and the goal is to find a vertex cover of minimum weight.
We extend the investigations carried out in \cite{Frank2013VertexCover} to the weighted minimum vertex cover problem. In~\cite{Frank2013VertexCover}, multi-objective models in combination with a simple multi-objective evolutionary algorithm called \gsemo are investigated. One key argument for the results presented for the (unweighted) vertex cover problem is that the population size is always upper bounded by $n+1$. This argument does not hold in the weighted case. Therefore, we study how a variant of \gsemo using appropriate diversity mechanisms is able to deal with the weighted vertex cover problem.

Our focus is on finding good approximations of an optimal solution. We analyse the time complexity with respect to $n$, $W_{max}$, and $OPT$, which denote the number of vertices, the maximum weight in the input graph, and the cost of the optimal solution respectively.
We first study the expected time until \gsemo has found a $2$-approximation in dependence of $n$ and $OPT$. Afterwards, we analyse the expected time of finding a solution with expected approximation ratio $(1+\varepsilon)$ for this problem when \gsemo uses the alternative mutation operator.
Furthermore, we consider \demo, a variant of \gsemo, which incorporates $\varepsilon$-dominance~\cite{Laumanns2002EpsilonDominance} as diversity mechanism. We show that \demo finds a $2$-approximation in expected polynomial time. Finally, we present a population-based approach that obtains a solution that has expected approximation ratio $(1+\varepsilon)$ in expected time $O(n\cdot 2^{\min \{n,2(1- \varepsilon)OPT \}} + n^3)$.

The outline of the paper is as follows. In Section~\ref{sec:prel}, the problem definition is presented as well as the classical \gsemo algorithm and \demo algorithm. Runtime analysis for finding a $2-$approximation and a $(1+\varepsilon)-$approximation by \gsemo is presented in Section~\ref{sec:3}. Section~\ref{sec:4} includes the analysis that shows \demo can find $2-$approximations of the optimum in expected polynomial time. The population-based algorithm is defined and investigated for finding a  $(1+\varepsilon)-$approximation in Section~\ref{sec:EpsilonApproxDEMO}. At the end, in Section~\ref{sec:Conclusion} we summarize and conclude.

\section{Preliminaries}
\label{sec:prel}
We consider the weighted vertex cover problem defined as follows.
Given a graph $G=(V,E)$ with vertex set $V=\{v_1,\ldots, v_n\}$ and edge set $E=\{e_1,\ldots, e_m\}$, and a positive weight function $w: V \rightarrow \mathds{N}^+$ on the vertices, the goal  is to find a subset of nodes, $V_C \subseteq V$, that covers all edges and has minimum weight, \ie $\forall e \in E, e \cap V_C \neq \emptyset$ and $\sum_{v\in V_C} w(v)$ is minimized. We consider the standard node-based approach, i.e. the search space is $\{0,1\}^n$ and for a solution $x = (x_1, \ldots, x_n)$ the node $v_i$ is chosen iff $x_i=1$.

The weighted vertex cover problem has the following Integer Linear Programming (ILP) formulation.
\begin{eqnarray*}
&& min \sum_{i=1}^n w(v_i)\cdot x_i \\
st. && x_i+x_j\geq 1 \ \ \ \forall (i,j)\in E \\
&& x_i\in \{0,1\}
\end{eqnarray*}

By relaxing the constraint $x_i\in \{0,1\}$ to $x_i\in [0,1]$, the linear program formulation of Fractional Weighted Vertex Cover is obtained. Hochbaum~\cite{LPForWeighteVC1983} has shown that we can find a 2-approximation using the LP result of the relaxed weighted vertex cover. This can be done by including any vertex $v_i$ for which $x_i\geq \frac{1}{2}$.

We consider primarily multi-objective approaches for the weighted vertex cover problem. Given a multi-objective fitness function $f = (f_1, \ldots, f_d) \colon S \rightarrow \mathds{R^d}$ where all $d$ objectives should be minimized, we have $f(x) \leq f(y)$ iff $f_i(x) \leq f_i(y)$, $1\leq i \leq d$. We say that $x$ (weakly) dominates $y$ iff $f(x) \leq f(y)$. Furthermore, we say that $x$ (strongly) dominates $y$ iff $f(x) \leq f(y)$ and $f(x) \not = f(y)$.

We now introduce the objectives used in our multi-objective evolutionary algorithm.
Let $G(x)$ be the graph obtained from $G$ by removing all nodes chosen by $x$ and the corresponding covered edges. Formally, we have $G(x) = (V(x), E(x))$ where $V(x) = V \setminus \{v_i \mid x_i =1\}$ and $E(x) = E \setminus \{ e \mid e \cap (V \setminus V(x)) \not = \emptyset\}$.
Kratsch and Neumann~\cite{Frank2013VertexCover} investigated a multi-objective baseline algorithm called \gsemo using the LP-value for $G(x)$ as one of the fitness values for the (unweighted) minimum vertex cover problem.

\begin{algorithm}[t]
  \caption{\gsemo}
   \label{alg:GlobalSemo}
    Choose $x\in \{0,1\}^n$ uniformly at random\;
    Determine $f(x)$\;
    $P\leftarrow \{x\}$\;
    \Repeat {termination condition satisfied}
    {
    Choose $x\in P$ uniformly at random\;
    Create $x'$ by flipping each bit $x_i$ of $x$ with probability $1/n$\;
    Determine $f(x')$\;
    \If {$\nexists y\in P \ \vert \  f(y)\leq f(x')$}
    {$P\leftarrow \{x'\}$\;
    delete all other solutions $z \in P$ with $f(x')\leq f(z)$ from $P$\;
	}
	}
\end{algorithm}

\begin{algorithm}[t]
  \caption{Alternative Mutation Operator}
   \label{alg:AltMutation}
    Choose $b\in \{0,1\}$ uniformly at random\;
    \eIf {$(b=1)$}
	{
	\ForEach {$i \in \{1, \cdots, n\}$}
	{ \eIf { $\exists j\in \{1, \cdots, n\} \ | \ \{v_i,v_j\}\in E(x)$}
		{Flip $x_i$ with probability $1/2$\;}
    		{Flip $x_i$ with probability $1/n$\;}
    	}
    	}
    	{
    	\ForEach {$i \in \{1, \cdots n\}$ }
    		{
    		Flip $x_i$ with probability $1/n$\;
    		}
    }
\end{algorithm}

Our goal is to expand the analysis on behaviour of multi-objective  evolutionary algorithms to the Weighted Vertex Cover problem. In order to do this, we modify the fitness function that was used in \gsemo in~\cite{Frank2013VertexCover}, to match the weighted version of the problem. We investigate the multi-objective fitness function $f(x)=(Cost(x), LP(x))$, where
\begin{itemize}
\item $Cost(x) = \sum_{i=1}^n w(v_i) x_i$ is the sum of weights of selected vertices
\item $LP(x)$ is the value of optimal solution of the LP for $G(x)$.
\end{itemize}

We analyse \gsemo with this fitness function using the standard mutation operator flipping each bit with probability $1/n$. We also investigate \gsemo using the alternative mutation operator introduced in~\cite{Frank2013VertexCover}  (see Algorithm~\ref{alg:AltMutation}). By this mutation operator, the nodes that are adjacent to uncovered edges are included with probability $1/2$ in some steps.

In the fitness function used in \gsemo, both $Cost(x)$ and $LP(x)$ can be exponential with respect to the input size; therefore, we need to deal with exponentially large number of solutions, even if we only keep the Pareto front. One approach for dealing with this problem is using the concept of $\varepsilon -$dominance \cite{Laumanns2002EpsilonDominance}. The concept of $\varepsilon -$dominance has previously been proved to be useful for coping with exponentially large Pareto fronts in some problems~\cite{Frank2008EpsilonDominance, Frank2008EpsilonDominance2}. Having two objective vectors $u=(u_1, \cdots, u_m)$ and $v=(v_1, \cdots, v_m)$, $u$ $\varepsilon-$dominates $v$, denoted by $u \preceq_\varepsilon v$, if for all $i\in \{1, \cdots, m\}$ we have $(1+\varepsilon)u_i \leq v_i$. In this approach, the objective space is partitioned into a polynomial number of boxes in which all solutions $\varepsilon-$dominate each other, and at most one solution from each box is kept in the population.

Motivated by this approach, \demo (Diversity Evolutionary Multi-objective Optimizer) has been investigated in~\cite{Frank2008EpsilonDominance2, DBLP:conf/ppsn/NeumannR08}. In Section~\ref{sec:4}, we analyze \demo (see Algorithm~\ref{alg:DEMO}) in which only one non-dominated solution can be kept in the population for each box based on a predefined criteria. In our setting, among two solutions $x$ and $y$ from one box, $y$ is kept in $P$ and $x$ is discarded if $Cost(y) +2\cdot LP(y)\leq Cost(x) +2\cdot LP(x)$.

\begin{algorithm}[t]
  \caption{DEMO}
   \label{alg:DEMO}
    Choose $x\in \{0,1\}^n$ uniformly at random\;
    Determine $b(x)$\;
    $P\leftarrow \{x\}$\;
    \Repeat {termination condition satisfied}
	{
    Choose $x\in P$ uniformly at random\;
    Create $x'$ by flipping each bit $x_i$ of $x$ with probability $1/n$\;
    Determine $f(x')$ and $b(x')$\;
    \eIf {$\exists y\in P \ \vert \  (f(y)\leq f(x') \wedge f(y)\neq f(x')) \vee (b(y)=b(x') \wedge Cost(y) +2\cdot LP(y)\leq Cost(x') +2\cdot LP(x'))$}
	{
		Go to 4\;
	}
    {$P\leftarrow \{x'\}$\;
    delete all other solutions $z \in P$ where $f(x')\leq f(z) \vee b(z)=b(x')$ from $P$\;
	}
	}
\end{algorithm}

To implement the concept of $\varepsilon-$dominance in \demo, we use the parameter $\delta=\frac{1}{2n}$ and define the boxing function $b: \{0,1\}^n \rightarrow \mathbb{N}^2$ as:
\begin{eqnarray*}
b_1(x)&=&\lceil \log_{1+\delta}(1+Cost(x)) \rceil, \\
b_2(x)&=&\lceil \log_{1+\delta}(1+LP(x)) \rceil,
\end{eqnarray*}

\ignore{
\textbf{Note that two boxes $B$ and $B'$ with $b_1(B)= b_1(B')$ and $b_2(B)< b_2(B')$ (or $b_2(B)= b_2(B')$ and $b_1(B)< b_1(B')$) can include search points that do not dominate each other; therefore, we may keep solutions from different boxes with same  values of $b_1$ or $b_2$. But if $b_1(B)< b_1(B')$ and $b_2(B)< b_2(B')$, then all search points in $B$ dominate all search points in $B'$. Hence, we define dominance among boxes as: box $B$ dominates box $B'$, denoted by $B<B'$, if $b_1(B)< b_1(B')$ and $b_2(B)< b_2(B')$. }
}

Analysing the runtime of our evolutionary algorithms, we are interested in the expected number of rounds of the repeat loop until a solution of desired quality has been obtained. We call this the expected time until the considered algorithm has achieved its desired goal.

\section{Analysis of \gsemo}
\label{sec:3}
In this section we analyse the expected time of \gsemo to find good approximations for the weighted vertex cover problem in dependence of the input size and OPT.
Before we present our analysis for \gsemo, we state some basic properties of the solutions in our multi-objective model.
The following theorem shown by Balinski~\cite{balinski1970maximum} states that all basic feasible solutions (or extremal points) of the fractional vertex cover LP are half-integral.

\begin{theorem}
\label{thm:onehalf}
Each basic feasible solution $x$ of the relaxed Vertex Cover ILP is half-integral, i.e., $x \in \{0,1/2,1\}^n$. \cite{balinski1970maximum}
\end{theorem}

As a result, there always exists a half integral optimal LP solution for a vertex cover problem. In several parts of this paper, we make use of this result. We establish the following two lemmata which we will use later on in the analysis of our algorithms.

\begin{lemma}
\label{lem:maxLP}
For any $x\in \{0,1\}^n$, $LP(x)\leq LP(0^n) \leq OPT$.
\end{lemma}
\begin{proof}
Let $y$ be the LP solution of $LP(0^n)$. Also, for any solution $x$, let $G(x)$ be the graph obtained from $G$ by removing all vertices chosen by $x$ and their edges. The solution $0^n$ contains no vertices; therefore, $y$ is the optimal fractional vertex cover for all edges of the input graph. Thus, for any solution $x$, $y$ is a (possibly non-optimal) fractional cover for $G(x)$; therefore, $LP(x)\leq LP(0^n)$.
Moreover, we  have $LP(0^n) \leq OPT$ as $LP(0^n)$ is the optimal value of the LP relaxation.
\end{proof}

\begin{lemma}
\label{lem:remainingLP}
Let $x=\{x_1,\cdots,x_n\}, x_i\in \{0,1\}$ be a solution and $y=\{y_1,\cdots,y_n\}, y_i\in[0,1]$ be a fractional solution for $G(x)$. If there is a vertex $v_i$ where $y_i\geq \frac{1}{2}$, mutating $x_i$ from 0 to 1 results in a solution $x'$ for which $LP(x')\leq LP(x)-y_i\cdot w(v_i)\leq LP(x)-\frac{1}{2}w(v_i)$.
\end{lemma}

\begin{proof}
The graph $G(x')$ is the same as $G(x)$ excluding the edges connected to $v_i$. Therefore, the solution $y'=\{y_1,\cdots, y_{i-1}, 0, y_{i+1},y_n\}$ is a fractional vertex cover for $G(x')$ and has a cost of $LP(x)-y_i w(v_i)$. The cost of the optimal fractional vertex cover of $G(x')$ is at most as great as the cost of $y'$; thus $LP(x')\leq LP(x)-y_iw(v_i) \leq LP(x)-\frac{1}{2}w(v_i)$.
\end{proof}

\subsection{2-Approximation}
\label{sec:2ApproxGSEMO}
We now analyse the runtime behaviour of \gsemo (Algorithm~\ref{alg:GlobalSemo}) with the standard mutation operator, in dependence of OPT.
We start by giving an upper bound on the population size of \gsemo.

\begin{lemma}
\label{lem:GSEMOPopSize}
The population size of Algorithm~\ref{alg:GlobalSemo} is upper bounded by $2 \cdot OPT + 1$.
\end{lemma}
\begin{proof}
For any solution $x$ there exists an optimal fractional vertex cover which is half-integral (Theorem~\ref{thm:onehalf}). Moreover, we are assuming that all the weights are integer values. Therefore, $LP(x)$ can only take $2LP(0^n)+1$ different values, because $LP(0^n)$ is an upper bound on $LP(x)$ (Lemma~\ref{lem:maxLP}). For each value of $LP$, only one solution is in $P$, because Algorithm~\ref{alg:GlobalSemo} keeps non-dominated solutions only. Therefore, the population size of this algorithm is upper bounded by $2 \cdot LP(0^n) + 1$ which is at most $2 \cdot OPT + 1$ due to Lemma~\ref{lem:maxLP}.
\end{proof}

For our analysis, we first consider the expected time of \gsemo to reach a population which contains the empty set of nodes. Once included, such a solution will never be removed from the population as it is minimal with respect to the cost function.

\begin{lemma}
\label{lem:Solution0}
The search point $0^n$ is included in the population in expected time of $O\left( OPT \cdot n ( \log W_{max} + \log n)\right)$.
\end{lemma}
\begin{proof}
From Lemma~\ref{lem:GSEMOPopSize} we know that the population contains at most $2\cdot OPT+1$ solutions. Therefore, at each step, there is a probability of $\frac{1}{2\cdot OPT+1}$ that the solution $x_{min}$ is selected where $Cost(x_{min})= \min_{x\in P} \ Cost(x)$.

If $Cost(x_{min}) > 0$, there must be $k\geq 1$ vertex such as $v_i$ in $x_{min}$ where $x_i=1$. Let $\Delta^t$ be the improvement that happens on the minimum cost in $P$ at step $t$. If all the 1-bits in solution $x_{min}$ flip to zero, at the same step or different steps, a solution $0^n$ will be obtained with $Cost(0^n)=0$, which implies that the expected improvement that flipping each 1-bit makes is $\Delta^t=\frac{Cost(x_{min})}{k}$ at each step $t$. Note that flipping 1-bits always improves the minimum cost and the new solution is added to the population. Moreover, flipping the 0-bits does not improve the minimum cost in the population and $x_{min}$ is not replaced with the new solution in that case.

At each step, with probability $\frac{1}{e}$ only one bit flips. With probability $\frac{k}{n}$, the flipping bit is a 1-bit, and makes an expected improvement of $\Delta^t=\frac{Cost(x_{min})}{k}$, and with probability $1-\frac{k}{n}$, a 0-bit is flipped with $\Delta^t=0$. We can conclude that the expected improvement of minimum cost, when only one bit of $x_{min}$ flips, is
$$\frac{k}{n}\cdot \frac{Cost(x_{min})}{k}= \frac{Cost(x_{min})}{n}$$

Moreover, the algorithm selects $x_{min}$ and flips only one bit with probability $\frac{1}{(2\cdot OPT+1) \cdot e}$; therefore, the expected improvement of minimum cost is
$$E[\Delta^t\mid x_{min}]\geq \frac{Cost(x_{min})}{(2\cdot OPT+1)\cdot e\cdot n}$$

The maximum value that $Cost(x_{min})$ can take is bounded by $W_{max} \cdot n$, and for any solution $x\neq 0^n$, the minimum value of $Cost(x)$ is at least 1.
Using  Multiplicative Drift Analysis~\cite{algorithmica/DoerrJW12} with $s_0\leq W_{max} \cdot n$ and $s_{min}\geq 1$, we can conclude that  in expected time $O\left( OPT \cdot n( \log W_{max} + \log n)\right)$ solution $0^n$ is included in the population.
\end{proof}

We now show that \gsemo is able to achieve a $2$-approximation efficiently as long as OPT is small.

\begin{theorem}
\label{thm:2ApprOPT}
The expected number of iterations of \gsemo until the population $P$ contains a two approximation is  $O(OPT \cdot n ( \log W_{max} + \log n))$.
\end{theorem}
\begin{proof}
Let $x$ be a solution that minimizes $LP(x)$ under the constraint that $Cost(x)+2\cdot LP(x)\leq 2\cdot OPT$. Note that this constraint holds for solution $0^n$ since $LP(0^n)\leq OPT$, and according to Lemma~\ref{lem:Solution0}, solution $0^n$ exists in the population in expected time of $O\left( OPT \cdot n ( \log W_{max} + \log n)\right)$.

If $LP(x)=0$, then all edges are covered and $x$ is a 2-approximate vertex cover, because we have $Cost(x)+2\cdot LP(x)\leq 2\cdot OPT$ as the constraint. Otherwise, some edges are uncovered and any LP solution of $G(x)$ assigns at least $\frac{1}{2}$ to at least one vertex of any uncovered edge. Let $y=\{y_1,\cdots,y_n\}$ be a basic LP solution for $G(x)$. According to Theorem~\ref{thm:onehalf}, $y$ is a half-integral solution.

Let $\Delta^t$ be the improvement that happens on the minimum $LP$ value among solutions that fulfil the constraint at time step $t$. Also, let $k$ be the number of nodes that are assigned at least $\frac{1}{2}$ by $y$. Flipping only one of these nodes by the algorithm happens with probability at least $\frac{k}{e\cdot n}$. According to Lemma~\ref{lem:remainingLP}, flipping one of these nodes, $v_i$, results in a solution $x'$ with $LP(x')\leq LP(x)-\frac{1}{2}w(v_i)$. Observe that the constraint of $Cost(x')+2\cdot LP(x') \leq 2\cdot OPT$ holds for solution $x'$. Therefore, $\Delta^t\geq y_i\cdot w(v_i)$, which is in expectation at least $\frac{LP(x)}{k}$ due to definition of $LP(x)$.
Moreover, at each step, the probability that $x$ is selected and only one of the $k$ bits defined above flips is $\frac{k}{(2\cdot OPT+1)\cdot e\cdot n}$. As a result we have:
$$E[\Delta^t\mid x]\geq \frac{k}{ (2\cdot OPT+1)\cdot e \cdot n}\cdot \frac{LP(x)}{k}= \frac{LP(x)}{en (2\cdot OPT+1)}$$

According to Lemma~\ref{lem:maxLP} for any solution $x$, we have $LP(x)\leq OPT$. We also know that for any solution $x$ which is not a complete cover, $LP(x)\geq 1$, because the weights are positive integers. Using the method of Multiplicative Drift Analysis~\cite{algorithmica/DoerrJW12} with $s_0\leq OPT$ and $s_{min}\geq 1$, in expected time of $O(OPT \cdot n \log OPT)$ a solution $y$ with $LP(y)=0$ and $Cost(y)+2LP(y)\leq 2OPT$ is obtained which is a 2-approximate vertex cover. Overall, since we have $OPT\leq W_{max}\cdot n$, the expected time of finding this solution is  $O(OPT \cdot n ( \log W_{max} + \log n))$.
\end{proof}

\subsection{Improved Approximations by Alternative Mutation}
\label{sec:EpsilonApproxGSEMO}
In this section, we analyse the expected time of \gsemo with alternative mutation operator to find a (1+$\varepsilon$)-approximation.

\begin{lemma}
\label{lem:Algorithm4containsimportantbitstring}
A solution $x$ fulfilling the two properties

\begin{enumerate}
\item  $LP(x) = LP(0^n) - Cost(x)$ and
\item  there is an optimal solution of the LP for G(x) which assigns 1/2 to each non-isolated vertex of $G(x)$
\end{enumerate}
 is included in the population of \gsemo in expected time $O(OPT \cdot n ( \log W_{max} + \log n+OPT))$.
\end{lemma}

\begin{proof}
As the standard mutation occurs with probability $1/2$ in the alternative mutation operator, the search point $0^n$ which satisfies property 1 is included in the population in expected time of $O( OPT \cdot n ( \log W_{max} + \log n))$ using the argument presented in the proof of Lemma~\ref{lem:Solution0}. Let $P' \subseteq P$ be a set of solutions such that for each solution $x \in P'$, $LP(x) + Cost(x) = LP(0^n)$. Let $x_{min} \in P'$ be a solution such that $LP(x_{min}) = min_{x \in P'} LP(x)$.

If the optimal fractional vertex cover for $G(x_{min})$ assigns 1/2 to each non-isolated vertex of $G(x_{min})$, then the conditions of the lemma hold. Otherwise, it assigns 1 to some non-isolated vertex, say $v$. The probability that the algorithm selects $x_{min}$ and flips the bit corresponding to $v$, is $\Omega (\frac{1}{OPT \cdot n})$, because the population size is $O(OPT)$ (Lemma~\ref{lem:GSEMOPopSize}). Let $x_{new}$ be the new solution. We have $Cost(x_{new})= Cost(x_{min})+w(v)$, and by Lemma~\ref{lem:remainingLP}, $LP_w(x_{new}) \leq LP_w(x_{min}) - w(v)$. This implies that $LP(x_{new}) + Cost(x_{new}) = LP(0^n)$; hence, $x_{new}$ is a Pareto Optimal solution and is added to the population $P$.

Since $LP_w(x_{min}) \leq OPT$ (Lemma~\ref{lem:maxLP}) and the weights are at least 1, assuming that we already have the solution $0^n$ in the population, by means of the method of fitness based partitions, we find the expected time of finding a solution that fulfils the properties given above as $O(OPT^2 \cdot n)$. Since the search point $0^n$ is included in expected time $O( OPT \cdot n ( \log W_{max} + \log n))$, the expected time that a solution fulfilling the properties given above is included in $P$ is $O(OPT \cdot n ( \log W_{max} + \log n +OPT))$.
\end{proof}

We now present the main approximation result for \gsemo using the alternative mutation operator.
\begin{theorem}
\label{thm:Epsilon}
The expected time until \gsemo has obtained a solution that  has expected approximation ratio $(1+ \varepsilon)$ is $O(OPT \cdot 2^{\min \{n,2(1- \varepsilon)OPT \}} + OPT \cdot n ( \log W_{max} + \log n+OPT))$.
\end{theorem}
\begin{proof}
By Lemma~\ref{lem:Algorithm4containsimportantbitstring}, a solution $x$ that satisfies the two properties given in Lemma~\ref{lem:Algorithm4containsimportantbitstring} is included in the population in expected time of $O(OPT \cdot n ( \log W_{max} + \log n +OPT))$. For a set of nodes, $X'$, we define $Cost(X')=\sum_{v\in X'} w(v)$.
Let $X$ be the vertex set of graph $G(x)$. Also, let $S \subseteq X$ be a vertex cover of $G(x)$ with the minimum weight over all vertex covers of $G(x)$, and $T$ be the set containing all non-isolated vertices in $X \setminus S$. Note that all vertices in $X \setminus (S \cup T)$ are isolated vertices in $G(x)$. Due to property 2 of Lemma~\ref{lem:Algorithm4containsimportantbitstring}, $\frac{1}{2} Cost(S) + \frac{1}{2} Cost(T) = LP(x) \leq Cost(S)$; therefore, $Cost(T) \leq Cost(S)$. Let $OPT' = OPT - Cost(x)$. Observe that $OPT' = Cost(S)$.

Let $s_1,\ldots,s_{|S|}$ be a numbering of the vertices in $S$ such that $w(s_i) \leq w(s_{i+1})$, for all $1 \leq i \leq |S|-1$. And let $t_1,\ldots,t_{|T|}$ be a numbering of the vertices in $T$ such that $w(t_i) \geq w(t_{i+1})$, for all $1 \leq i \leq |T|-1$.
Let $S_1 = \{s_1,s_2,\ldots,s_{\rho}\}$, where $\rho = min \{|S|, (1- \varepsilon) \cdot OPT'\}$, and $T_1 = \{t_{1}, t_2, \ldots, t_{\eta}\}$, where $\eta = min \{|T|, (1- \varepsilon) \cdot OPT'\}$.

With probability $\Omega (\frac{1}{OPT})$, the algorithm \gsemo selects the solution $x$, and sets  $b=1$ in the Alternative Mutation Operator.
With $b=1$, the probability that the bits corresponding to all vertices of $S_1$ are flipped, is $\Omega((\frac{1}{2})^{\rho})$, and the probability that none of the bits corresponding to the vertices of $T_1$ are flipped is $\Omega((\frac{1}{2})^{\eta})$. Also, the bits corresponding to the isolated vertices of $G(x)$ are flipped with probability $\frac{1}{n}$ by the Alternative Mutation Operator; hence, the probability that none of them flips is $\Omega (1)$. As a result, with probability $\Omega(\frac{1}{OPT} \cdot (\frac{1}{2})^{\rho + \eta})$, solution $x$ is selected, the vertices of $S_1$ are included, and the vertices of $T_1$ and isolated vertices are not included in the new solution $x'$. Since $\rho+\eta \leq 2(1- \varepsilon) \cdot OPT' \leq 2(1- \varepsilon) \cdot OPT$, and also $\rho+\eta \leq n$; the expected time until solution $x'$ is found after reaching solution $x$, is $O(OPT \cdot 2^{\min \{n,2(1- \varepsilon)OPT \}})$.

Note that the bits corresponding to vertices of $S_2 = S \setminus S_1$ and $T_2 = T \setminus T_1$, are arbitrarily flipped in solution $x'$ with probability $1/2$ by the Alternative Mutation Operator. Here we show that for the expected cost and the LP value of $x'$, the following constraint holds: $E[Cost(x')]+ 2 \cdot LP(x')\leq (1+\varepsilon)\cdot OPT$.

Let $S' \subseteq S$ and $T' \subseteq T$ denote the subset of vertices of $S$ and $T$  that are actually included in the new solution $x'$ respectively. In the following, we show that for the expected values of $Cost(S')$ and $Cost(T')$, we have:
\begin{equation}
\label{inequality1}
E \left[Cost(S')\right] \geq (1-\varepsilon) \cdot OPT' + E \left[Cost(T')\right]
\end{equation}

Since the bits corresponding to the vertices of $S_2$ and $T_2$ are flipped with probability 1/2, for the expected values of $Cost(S')$ and $Cost(T')$ we have:
\begin{eqnarray*}
E\left[Cost( S')\right] &=& Cost( S_1) + \frac{Cost( S_2)}{2} \\
&=& Cost( S_1) + \frac{Cost( S) - Cost( S_1)}{2}\\
 &=& 1/2Cost( S) + 1/2 Cost(S_1)
\end{eqnarray*}
and
\begin{eqnarray*}
E\left[Cost( T')\right] &= &1/2 Cost(T_2)
\end{eqnarray*}

If $\rho = |S|$, then $S_1 = S$ and $Cost( S_1) = Cost(S) = OPT'$. If $\rho = (1-\varepsilon) \cdot OPT'$, we have $Cost(S_1) \geq (1-\varepsilon) \cdot OPT'$, since each vertex has a weight of at least 1. 
Using $Cost( S) = OPT'$ and the inequality above, we have
\begin{eqnarray*}
E\left[Cost( S')\right] &\geq & (1-\varepsilon) \cdot OPT' + \frac{\varepsilon \cdot OPT'}{2}
\end{eqnarray*}
We divide the analysis into two cases based on the relation between $\eta$ and $|T|$.

Case (I). $\eta = |T|$. Then $T_2  = T' = \emptyset$. Thus, $E\left[Cost( T')\right] = 0$ and Inequality (\ref{inequality1}) holds true.

Case (II). $\eta = (1-\varepsilon) \cdot OPT' < |T|$. Since $w(t_i) \geq w(t_{i+1})$ for $1 \leq i \leq |T|-1$ and $Cost( T) \leq Cost(S) = OPT'$, we have
\begin{eqnarray*}
Cost( T_2) &\leq& \frac{|T| - \eta}{|T|} Cost( T) \\
&\leq& \frac{OPT' - (1-\varepsilon) \cdot OPT'}{OPT'} Cost( T)\\
&\leq& \varepsilon Cost( S) = \varepsilon \cdot OPT'
\end{eqnarray*}
Thus for the expected value of $Cost( T')$, we have
$$E\left[Cost( T')\right] = \frac{1}{2} Cost( T_2) \leq \frac{\varepsilon \cdot OPT'}{2}$$

Summarizing above analysis, we can get that the Inequality~\ref{inequality1} holds. In the following, using Inequality (\ref{inequality1}), we prove that, on expectation, the new solution $x'$ satisfies the inequality $Cost(x') + 2 \cdot LP(x') \leq (1+\varepsilon) \cdot OPT$. 
$$E\left[Cost(x')\right]+ 2 \cdot LP(x')$$
$$= Cost(x) + E\left[Cost( S')\right] + E\left[Cost( T')\right] + 2 \cdot LP(x')$$
$$\leq Cost(x) + E\left[Cost( S')\right] + E\left[Cost( S')\right] - (1-\varepsilon) \cdot OPT' + 2 \cdot LP(x')$$
$$\leq Cost(x) + 2 E\left[Cost( S')\right] - (1-\varepsilon) \cdot OPT' + 2 \cdot (OPT' - E\left[ Cost( S')\right])$$
$$= Cost(x) + (1+\varepsilon) \cdot OPT' = Cost(x) + (1+\varepsilon) \cdot (OPT - Cost(x)) $$
$$\leq  (1+\varepsilon) \cdot OPT.$$

Now we analyze whether the new solution $x'$ could be included in the population $P$. If $x'$ could not be included in $P$, then there is a solution $x''$ dominating $x$, i.e., $LP(x'') \leq LP(x')$ and $Cost(x'') \leq Cost(x')$. This implies $Cost(x'') + 2 \cdot LP(x'') < Cost(x') + 2 \cdot LP(x') \leq (1+\varepsilon) \cdot OPT$.
Therefore, after having a solution that fulfils the properties of Lemma~\ref{lem:Algorithm4containsimportantbitstring} in $P$, in expected time $O(OPT \cdot 2^{\min \{n,2(1- \varepsilon)OPT \}})$, the population would contain a solution $y$ such that $Cost(y) + 2 \cdot LP(y) \leq (1+\varepsilon) \cdot OPT$. 

Let $P'$ contain all solutions $x\in P$ such that $Cost(x) + 2 \cdot LP(x) \leq (1+\varepsilon) \cdot OPT$, and let $x_{min}$ be the one that minimizes $LP$. With similar proof as we saw in Theorem~\ref{thm:2ApprOPT} it is possible to show that at each step, $LP(x_{min})$ improves by $\frac{LP(x)}{en (2\cdot OPT+1)}$ in expectation. Using Multiplicative Drift Analysis, we get the expected time $O(OPT\cdot n\log OPT)$ to find a solution $y$ for which $LP(y)=0$ and $Cost(y) + 2 \cdot LP(y) \leq (1+\varepsilon) \cdot OPT$.

 Overall, the expected number of iterations of \gsemo with alternative mutation operator, for getting a $(1+\varepsilon)$-approximate weighted vertex cover, is bounded by $O(OPT \cdot 2^{\min \{n,2(1- \varepsilon)OPT \}} + OPT \cdot n ( \log W_{max} + \log n+OPT))$.
\end{proof}

\section{Analysis of \demo}
\label{sec:4}
Due to Lemma~\ref{lem:GSEMOPopSize}, with \gsemo, the population size is upper bounded by $O(OPT)$, which can be exponential in terms of the input size. In this section, we  analyse the other evolutionary algorithm, \demo (Algorithm~\ref{alg:DEMO}), that uses some diversity handling mechanisms for dealing with exponentially large population sizes.
The following lemmata are used in the proof of Theorem~\ref{thm:2ApproxBox}.

\begin{lemma}
\label{lem:popSize}
Let $W_{max}$ be the maximum weight assigned to a vertex. The population size of \demo is upper bounded by $ O\left(n\cdot(\log n +\log W_{max})\right)$.
\end{lemma}
\begin{proof}
The values that can be taken by $b_1$ are integer values between 0 and $\lceil \log_{1+\delta}(1+Cost(1^n)) \rceil$ and the values that can be taken by $b_2$ are integer values between 0 and $\lceil \log_{1+\delta}(1+LP(0^n)) \rceil$ (Lemma~\ref{lem:maxLP}). Since $n\cdot W_{max}$ is an upper bound for both $Cost(1^n)$ and $LP(0^n)$, the number of rows and also the number of columns are bounded by
\begin{eqnarray*}
k &=& \left(1+\lceil\log_{1+\delta}(1+n\cdot W_{max})\rceil\right) \\
&\leq & \left(1+\lceil\frac{\log (1+n\cdot W_{max})}{\log(1+\delta)}\rceil\right)\\
&=& O\left(n\cdot(\log n +\log W_{max})\right)
\end{eqnarray*}
The last equality holds because $\delta= \frac{1}{2n}.$

We here show that the size of the population is  $P_{size}\leq 2k-1$.
Since the dominated solutions according to $f$ are discarded by the algorithm, none of the solutions in $P$ can be located in a box that is dominated by another box that contains a solution in $P$. Moreover, at most one solution from each box is kept in the population; therefore, $P_{size}$ is at most the maximum number of boxes where none of them dominates another.

Let $k_1$ be the number of boxes that contain a solution of $P$ in the first column. Let $r_1$ be the smallest row number among these boxes. Observe that $r_1\leq k-k_1+1$ and the equality holds when the boxes are from rows $k$ down to $k-k_1+1$. Any box in the second column with a row number of $r_1+1$ or above is dominated by the box of the previous column and row $r_1$. Therefore, the maximum row number for a box in the second column, that is not dominated, is $r_1\leq k-k_1+1$. With generalizing the idea, the maximum row number for a box in the column $i$, that is not dominated, is $r_{i-1}\leq k-k_1-\cdots -k_{i-1}+i-1$, where for $1\leq j\leq k$, $k_j$ is the number of boxes that contain a solution of $P$ in column $j$.

The last column has $k_k\leq r_{k-1}$ boxes which gives us:
$$k_k\leq r_{k-1}\leq k-k_1-\cdots -k_{k-1}+k-1$$
This implies that
$$k_1+\cdots +k_k\leq r_{k-1}\leq 2k-1$$
which completes the proof.
\end{proof}

\begin{lemma}
\label{lem:solution0box}
The search point $x_z=0^n$ is included in the population in expected time of $O(n^3(\log n +\log W_{max})^2)$.
\end{lemma}
\begin{proof}
From Lemma~\ref{lem:popSize} we know that the population contains $P_{size}=O\left(n\cdot(\log n +\log W_{max})\right)$ solutions. Therefore, at each step, there is a probability of at least $\frac{1}{p_{size}}$ that the solution $x_{min}$ is selected where $b_1(x_{min})= \min_{x\in P} \ b_1(x)$.

If $b_1(x_{min})=0$, we have $Cost(x_{min})=0$, which means $x_{min}=0^n$ since the weights are greater than 0.

If $b_1(x_{min})\neq 0$, there must be at least one vertex $v_i$ in $x_{min}$ where $x_i=1$. Consider $v_j$ the vertex that maximizes $w(v_i)$ among vertices $v_i$ where $x_i=1$. If $Cost(x)= C$, then $w(v_j)\geq \frac{C}{n}$, because $n$ is an upper bound on the number of vertices selected by $x_{min}$. As a result, removing vertex $x_j$ from solution $x_{min}$ results in a solution $x'$ for which $Cost(x')\leq C\cdot(1-\frac{1}{n})$.
Using this value of $Cost(x')$, we have
\begin{eqnarray*}
(1+\delta)(1+Cost(x')) &\leq & 1+\delta+ C(1-\frac{1}{n})(1+\delta)\\
&\leq & 1+\delta+ C + C(\delta-\frac{1}{n}-\frac{\delta}{n}) \\
&\leq & 1+C\delta+ C + C(\delta-\frac{1}{n}-\frac{\delta}{n}) \\
&\leq & 1+ C + C(2\delta-\frac{1}{n}-\frac{\delta}{n}) \\
&\leq & 1+ C
\end{eqnarray*}

The third inequality above holds because $C\geq 1$ and the last one holds because $\delta=\frac{1}{2n}$. From $(1+\delta)(1+Cost(x'))\leq 1+C$ we can observe that
$$1+\log_{1+\delta}(1+Cost(x'))\leq \log_{1+\delta}(1+C)$$
which implies $b_1(x')\leq b_1(x)-1$.
Note that $x'$ is obtained by performing a 1-bit flip on $x$ and is done at each step with a probability of at least
$$\frac{1}{P_{size}}\cdot\ \frac{1}{n}\cdot (1-\frac{1}{n})^{n-1} $$
$$= \Omega\left( \frac{1}{n(\log n +\log W_{max})}\cdot\frac{1}{n} \right)$$

Therefore, in expected time of at most $O\left(n^2(\log n +\log W_{max})\right)$ the new solution, $x'$ is obtained which is accepted by the algorithm because it is placed in a box with a smaller value of $b_1$ than all solutions in $P$ and hence not dominated. There are $O\left(n(\log n +\log W_{max})\right)$ different values for $b_1$; therefore, the solution $x_z=0^n$ with $b_1(x_z)=0$ is found in expected time of at most  $O\left(n^3(\log n +\log W_{max})^2\right)$.
\end{proof}

\ignore{
\begin{lemma}
\label{lem:NewOnesFulfilTheConstraint}
If $x$ is a solution for which $Cost(x) +2\cdot LP(x) \leq 2\cdot OPT$ holds, and $v_i$ is a vertex with  $x_i\geq \frac{1}{2}$ in the LP solution for $G(x)$, then adding $v_i$ to $x$ results in a solution $x'$ for which $Cost(x') +2\cdot LP(x') \leq 2\cdot OPT$ holds.
\end{lemma}
\begin{proof}
Since solutions $x$ and $x'$ are only different in one vertex, $v_i$, we have $Cost(x')=Cost(x)+w(v_i)$. Moreover, according to Lemma~\ref{lem:remainingLP}, $LP(x')\leq LP(x)- \frac{1}{2}\cdot w(v_i)$. Therefore,
$$Cost(x') +2\cdot LP(x') \leq Cost(x)+w(v_i) +2\left(LP(x)- \frac{w(v_i)}{2}\right)$$
$$\leq  Cost(x) +2\cdot LP(x) \leq 2\cdot OPT  $$
which completes the proof.\end{proof}
}

\begin{lemma}
\label{lem:OneBitFlipAndLessB2}
Let $x\in P$ be a search point such that $Cost(x) +2\cdot LP(x) \leq 2\cdot OPT$ and $b_2(x)>0$. There exists a 1-bit flip leading to a search point $x'$ with $Cost(x') +2\cdot LP(x') \leq 2\cdot OPT$ and $b_2(x') < b_2(x)$.
\end{lemma}
\begin{proof}
Let $y=\{y_1\cdots y_n\}$ be a basic half integral LP solution for $G(x)$. Since $b_2(x)=LP(x)\neq 0$, there must be at least one uncovered edge; hence, at least one vertex $v_i$ has a $y_i\geq\frac{1}{2}$ in LP solution $y$. Consider $v_j$ the vertex that maximizes $y_i w(v_i) $ among vertices $v_i, \ 1\leq i\leq n$. Also, let $x'$ be a solution obtained by adding $v_j$ to $x$.
Since solutions $x$ and $x'$ are only different in one vertex, $v_j$, we have $Cost(x')=Cost(x)+w(v_j)$. Moreover, according to Lemma~\ref{lem:remainingLP}, $LP(x')\leq LP(x)- \frac{1}{2}\cdot w(v_j)$. Therefore,
$$Cost(x') +2\cdot LP(x') \leq Cost(x)+w(v_j) +2\left(LP(x)- \frac{w(v_j)}{2}\right)$$
$$\leq  Cost(x) +2\cdot LP(x) \leq 2\cdot OPT  $$
which means solution $x'$ fulfils the mentioned constraint. If $LP(x)= W$, then $y_jw(v_j)\geq \frac{W}{n}$, because $n$ is an upper bound on the number of vertices selected by the LP solution. As a result, using Lemma~\ref{lem:remainingLP}, we get $LP(x')\leq W\cdot(1-\frac{1}{n})$.
Therefore, with similar analysis as Lemma~\ref{lem:solution0box} we get:
\begin{eqnarray*}
(1+\delta)\left(1+LP(x')\right) &\leq & 1+\delta+ W\left(1-\frac{1}{n}\right)(1+\delta)\\
&\leq & 1+ W
\end{eqnarray*}

This inequality implies
$$1+\log_{1+\delta}(1+LP(x'))\leq \log_{1+\delta}(1+W)$$
 As a result, $b_2(x') < b_2(x)$ holds for $x'$, which is obtained by performing a 1-bit flip on $x$, and the lemma is proved.
\end{proof}

\begin{theorem}
\label{thm:2ApproxBox}
The expected time until \demo constructs a 2-approximate vertex cover is $O\left(n^3\cdot(\log n +\log W_{max})^2\right)$.
\end{theorem}
\begin{proof}
Consider solution $x\in P$ that minimizes $b_2(x)$ under the constraint that $Cost(x) +2\cdot LP(x) \leq 2\cdot OPT$. Note that $0^n$ fulfils this constraint and according to Lemma~\ref{lem:solution0box}, the solution $0^n$ will be included in $P$ in time $O\left(n^3(\log n +\log W_{max})^2\right)$.

If $b_2(x)=0$ then $x$ covers all edges and by selection of $x$ we have $Cost(x)\leq 2\cdot OPT$, which means that $x$ is a $2-$approximation.

In case $b_2(x)\neq 0$, according to Lemma~\ref{lem:OneBitFlipAndLessB2} there is a one-bit flip on $x$ that results in a new solution $x'$ for which $b_2(x')< b_2(x)$, while the mentioned constraint also holds for it. Since the population size is $O\left(n\cdot(\log n +\log W_{max})\right)$ (Lemma~\ref{lem:popSize}), this 1-bit flip happens with a probability of $\Omega\left(n^{-2}\cdot(\log n +\log W_{max})^{-1}\right)$ and $x'$ is obtained in expected time of $O(n^3\cdot(\log n +\log W_{max})^2)$.  This new solution will be added to $P$ because a solution $y$ with $Cost(y) +2\cdot LP(y) > 2\cdot OPT$ can not dominate $x'$ with $Cost(x') +2\cdot LP(x') \leq 2\cdot OPT$, and $x'$ has the minimum value of $b_2$ among solution that fulfil the constraint. Moreover, if there already is a solution, $x_{prev}$, in the same box as $x'$, it will be replaced by $x'$ because $Cost(x_{prev}) +2\cdot LP(x_{prev})>2\cdot OPT$; otherwise, it would have been selected as $x$.

There are at most $1+\lceil \frac{\log n +\log W_{max}}{\log(1+\delta)} \rceil$ different values for $b_2$ in the objective space, therefore, the expected time until a solution $x''$ is found so that $b_2(x'')=0$ and $Cost(x'') +2\cdot LP(x'') \leq 2\cdot OPT$, is at most $O(n^3\cdot(\log n +\log W_{max})^2)$.
\end{proof}

\section{Diverse Population-based EA}
\label{sec:EpsilonApproxDEMO}
In this section, we introduced a population-based algorithm (see Algorithm~\ref{alg:IMP}) that keeps for each $k$, $0 \leq k \leq n$, at most two solutions. This implies that the population size is upper bounded by $2n$. The two solutions kept in the population are chosen according to  different weighing of the cost and the LP-value.  For each solution $x$, let $|x|_1$ be the number of selected nodes in $x$. Algorithm~\ref{alg:IMP} keeps a new solution $x'$ in the population, if it minimizes $Cost(z)+LP(z)$ or $Cost(z)+2\cdot LP(z)$ among other solutions $x \in P$ where $|x|_1=|x'|_1$. Algorithm~\ref{alg:IMP} gives a detailed description.

\begin{algorithm}[t]
  \caption{\popBased}
   \label{alg:IMP}
    Choose $x\in \{0,1\}^n$ uniformly at random\;
    $P\leftarrow \{x\}$\;
    \Repeat {termination condition satisfied}
	{
    Choose $x\in P$ uniformly at random\;
    Create $x'$ by using Alternative Mutation Operator\;
    $P\leftarrow \{x'\}$\;
    Let $P'$ be a set containing all solutions $y\in P$ where $|y|_1=|x'|_1$\;
    Find solutions $y_{min_1}$ and $y_{min_2}$ from $P'$ such that
    $y_{min_1}$ minimizes $Cost(z)+LP(z)$, and
    $y_{min_2}$ minimizes $Cost(z)+2\cdot LP(z)$ among solutions $z \in P'$\;
    $P = P\setminus P'$\;
    $P \leftarrow \{y_{min_1} ,y_{min_2}\}$\;
	}
\end{algorithm}

\ignore{
\begin{lemma}
\label{lem:ImpPopulationSize}
The population size of \popBased is upper bounded by $2n$.
\end{lemma}
\begin{proof}
Observe that for each integer value $c$, where $0 \leq c \leq n$, there are at most two solutions $x_1$ and $x_2$ with $|x_1|_1 = |x_2|_1 = c$ in the population.    Thus the size of the population, $p_{size}$ is upper bounded by $2n$.
\end{proof}
}

Taking into account that the population size is upper bounded by $2n$ and considering in each step an individual with the smallest number of ones in the population for mutation, one can obtain the following lemma by standard fitness level arguments.
\begin{lemma}
\label{lem:ImpIncludesAllZeroBitstring}
The search point $0^n$ is included in the population in expected time of $O(n^2 \log n)$.
\end{lemma}
\ignore{
\begin{proof}
The probability  of selecting $x_{min}$ where $|x_{min}|_1 = \min_{x\in P} \ |x|_1$ is at least $1/2n$ as the population size is upper bounded by $2n$..

If $|x_{min}|_1= 0$, then the solution $0^n$ already exists in the population. Otherwise, $|x_{min}|_1= k>0$, and there exist $k$ 1-bit flips on $x_{min}$ that results in a solution $x'$ for which $|x'|_1 = |x_{min}|_1 -1$. Solution $x'$ is added to $P$, because, by selection of $x_{min}$, there is no solution $y \in P$ with $|y|_1 = |x'|_1$ to compete with $x'$. The probability that $x_{min}$ is selected and the proper mutation happens is
$\Omega\left(\frac{k}{n^2}\right)$. Therefore, in expected time $O(n^2 \log n)$ a solution $z$ with $|z|_1=0$ is found by the algorithm.
\end{proof}
}

To show the main result for \popBased, we will use the following lemma.

\begin{lemma}
\label{lem:Lemma6}
A solution $x$ fulling the two properties

\begin{enumerate}
\item  $LP(x) = LP(0^n) - Cost(x)$ and
\item  there is an optimal solution of the LP for G(x) which assigns 1/2 to each non-isolated vertex of $G(x)$
\end{enumerate}
 is included in the population of the \popBased in expected time $O(n^3)$.
\end{lemma}

\begin{proof}
By Lemma~\ref{lem:ImpIncludesAllZeroBitstring}, solution $0^n$ is contained in the population in expected time $O(n^2 \log n)$, which satisfies the property 1 given above. Let $P' \subseteq P$ be a set containing all solutions in $P$ that satisfy the property 1 given above.

Let $x_{max}$ be the solution of $P'$ with the  maximal number of $1$-bits. If the optimal fractional vertex cover for $G(x_{max})$ assigns 1/2 to each non-isolated vertex of $G(x_{max})$, then the second property also holds. If the optimal fractional vertex cover for $G(x_{max})$ assigns 1 to some non-isolated vertex, say $v$, then
the algorithm selects $x_{max}$ and flips exactly the bit corresponding to $v$ with probability $\Omega (\frac{1}{n^2})$. Let $x'$ be the new solution. By selection of $x_{max}$ we know that $x'$ is the only solution with $|x_{max}|_1+1$ one-bits; hence, added to $P$.

Since the maximum value of $|x|_1$ is $n$, after expected time of $O(n^3)$, there is a solution in the population that fulfils the properties given in the lemma.
\end{proof}

We now show the main result for the \popBased.

\begin{theorem}
\label{thm:ImprovedEpsilon}
The expected time until \popBased has obtained a solution that  has expected approximation ratio $(1+ \varepsilon)$ is $O(n\cdot 2^{\min \{n,2(1- \varepsilon)OPT \}} + n^3)$.
\end{theorem}
\begin{proof}
By Lemma~\ref{lem:Lemma6} we know that after expected time of $O(n^3)$, there is a solution, $x$, in the population that fulfils the properties given in  that lemma. With analysis similar to what we had in Theorem~\ref{thm:Epsilon}, we can show that a solution $x$ with
$Cost(x) + 2 \cdot LP(x) \leq (1+\varepsilon) \cdot OPT$ is produced in expected time $O(n\cdot 2^{\min \{n,2(1- \varepsilon)OPT \}} + n^3)$.

Now we see whether solution $x$ is added to population $P$. If $x$ could not be added to $P$, then there exists a solution $y \in P$ such that $|y|_1 = |x|_1$ and $Cost(y) + 2 \cdot LP(y) \leq Cost(x) + 2 \cdot LP(x)$. Thus, the population  already includes a solution $y$ such that $Cost(y) + 2 \cdot LP(y) \leq (1+\varepsilon) \cdot OPT$.

Let $P'$ be a set containing all solutions  $x \in P$ such that $Cost(x) + 2 \cdot LP(x) \leq (1+\varepsilon) \cdot OPT$. Let $x_{max} \in P'$  such that $|x_{max}|_1 = \max_{x \in P'} |x|_1$.

If $LP(x_{max}) = 0$, then solution $x_{max}$ leads to a vertex cover for graph $G$. If $LP(x_{max}) > 0$, we present a way to construct a $(1+\varepsilon)$-approximate vertex cover as follows, using $x_{max}$. If $LP(x_{max}) > 0$, then there exists at least one vertex $v$ to which the optimal fractional vertex cover $LP(x_{max})$ assigns value at least 1/2. Then the algorithm selects the solution $x_{max}$ and flips exactly the bit corresponding to the vertex $v$ with probability $\Omega (\frac{1}{n^2})$.
Let $y$ be the new solution. We have
$$Cost(y) + 2 \cdot LP(y) \leq Cost(x_{max}) + 2 \cdot LP(x_{max}) \leq (1+\varepsilon) \cdot OPT.$$

Suppose that $y$ could not be included in $P$, then there exists a solution $y'$ in $P$ such that $|y'|_1 = |y|_1$ and $2 \cdot LP(y') + Cost(y') \leq 2 \cdot LP(y) + Cost(y) \leq (1+\varepsilon) \cdot OPT$, which contradicts the assumption that $|x_{max}|_1 = \max_{x \in P'} |x|_1$. Therefore, solution $y$ could be included in $P$.

Observe that for any solution $x$, if $|x|_1 = n$, then $LP(x)=0$. Thus, after expected time of at most $O(n^3)$, the population $P$ could include a solution $y$ such that $Cost(y) + 2 \cdot LP(y) \leq (1+\varepsilon) \cdot OPT$ and $LP(y) = 0$, which is a $(1+\varepsilon)$-approximate weighted vertex cover.

Overall, the expected time in which \popBased finds a $(1+\varepsilon)$-approximate weighted vertex cover, is bounded by $O(n\cdot 2^{\min \{n,2(1- \varepsilon)OPT \}} + n^3)$.
\end{proof}

\section{Conclusion}
\label{sec:Conclusion}
The minimum vertex cover problem is one of the classical NP-hard combinatorial optimization problems. In this paper, we have generalized previous results of Kratsch and Neumann~\cite{Frank2013VertexCover} for the unweighted minimum vertex cover problem to the weighted case where in addition weights on the nodes are given. Our investigations show that \gsemo efficiently computes a $2$-approximation as long as the value of an optimal solution is small. Furthermore, we have studied the algorithm \demo using the $\varepsilon$-dominance approach and shown that it reaches a $2$-approximation in expected polynomial time. Furthermore, we have generalized the results for \gsemo to $(1+\varepsilon)$-approximations and presented a population-based approach with a specific diversity mechanism that reaches an $(1+\varepsilon)$-approximation in expected time $O(n\cdot 2^{\min \{n,2(1- \varepsilon)OPT \}} + n^3)$.

\section*{Acknowledgements}
This research has been supported by Australian Research Council grants DP140103400 and DP160102401.

\end{document}